\documentclass{article} % For LaTeX2e
\usepackage{iclr2027_conference,times}

\usepackage[utf8]{inputenc} % allow utf-8 input
\usepackage[T1]{fontenc}    % use 8-bit T1 fonts
\usepackage{hyperref}       % hyperlinks
\usepackage{url}            % simple URL typesetting
\usepackage{booktabs}       % professional-quality tables
\usepackage{amsfonts}       % blackboard math symbols
\usepackage{nicefrac}       % compact symbols for 1/2, etc.
\usepackage{microtype}      % microtypography
\usepackage{xcolor}         % colors

\usepackage{amsmath}
\usepackage{algorithm}
\usepackage{algpseudocode}
\usepackage{graphicx}
\usepackage{multirow}
\usepackage{wrapfig}
\usepackage{makecell}

\usepackage{amssymb}
\newtheorem{theorem}{Theorem}[section]

\newtheorem{lemma}[theorem]{Lemma}

\newtheorem{proposition}{Proposition}[section]

\usepackage{listings}
\usepackage{xcolor}

\lstdefinestyle{promptstyle}{
  basicstyle=\ttfamily\footnotesize,
  backgroundcolor=\color{gray!10},
  frame=single,
  breaklines=true,
  columns=fullflexible,
  showstringspaces=false
}

\newcommand{\ie}{\textit{i}.\textit{e}.}

\usepackage{amsmath,amsfonts,bm}

\def\eqref#1{equation~\ref{#1}}
\def\1{\bm{1}}

\DeclareMathAlphabet{\mathsfit}{\encodingdefault}{\sfdefault}{m}{sl}
\SetMathAlphabet{\mathsfit}{bold}{\encodingdefault}{\sfdefault}{bx}{n}

\DeclareMathOperator*{\argmin}{arg\,min}

\usepackage{hyperref}
\usepackage{url}

\title{Cooperative Multi-Agent Vision-Language-Action Models via Reinforced Fine Tuning}

\author{
\textbf{
Ruixiao Xu$^{1}$,
WONG Lik Hang Kenny$^{2}$,
Zhiqian Liu$^{1}$,
Jianing Guo$^{1,3}$,
Hanxiao Li$^{1}$, } \\
\textbf{
Kejian Shi$^{2}$,
Shuning Zhang$^{4}$,
Pu Feng$^{5}$,
Yongjia Ma$^{6}$,
Yuqing Ma$^{1}$,
Kai Chen$^{2}$,
Qi Dou$^{2}$, } \\
\textbf{
Yaodong Yang$^{3,7}$,
Xianglong Liu$^{1,5}$,
Simin Li$^{1,2}$ \thanks{Corresponding Author. E-mails: lisiminsimon@buaa.edu.cn}
}
\\[1em]
$^{1}$Beihang University,
$^{2}$The Chinese University of Hong Kong, 
$^{3}$PKU-Psibot Lab, 
\\
$^{4}$Tsinghua University, 
$^{5}$Zhongguancun Laboratory, 
$^{6}$Li Auto Inc., 
$^{7}$Peking University
}

\iclrfinalcopy % Uncomment for camera-ready version, but NOT for submission.
\begin{document}

\maketitle

\begin{abstract}
We study reinforcement learning (RL) methods for cooperative multi-agent Vision-Language-Action (VLA) models. This problem is challenging because VLAs are pretrained on large-scale single-agent data and therefore lack the fine-grained coordination skills required for inter-robot collaboration. Supervised fine-tuning (SFT) on multi-robot demonstrations partially bridges this gap, but its performance is bounded by the demonstration data and cannot improve from its own experience. We present a three-stage reinforced fine-tuning (RFT) pipeline for multi-agent VLAs. First, initialization-aware data collection sweeps over initial configurations and invokes human demonstrations only when the pretrained VLA repeatedly fails, yielding robustness to initialization shift with reduced human cost. Second, offline credit-filtered tuning assigns credit to individual agents and fine-tunes on per-agent trajectories with positive advantage rather than on entire joint rollouts. Third, we find existing online RL for VLAs are less effective for hard multi-agent tasks, which we attribute to noisy co-exploration and unstable updates. We instead use online latent-space fine tuning, which freeze the VLA and perform RL in its latent noise space. 
We evaluate our multi-agent VLA with both $\pi_0$ and $\pi_{0.5}$ backbones across 11 tasks in RoboTwin, RoboFactory and real-world manipulation with two Franka robots. Our multi-agent VLA improves the average success rate by $+23.1\%$, $+16.4\%$, and $+44\%$ on RoboTwin, RoboFactory, and real-world tasks, respectively. Code available at \url{https://anonymous.4open.science/r/mavla_rft-2BC0/}.
\end{abstract}

\section{Introduction}
\label{sec:introduction}

% 1. VLA很通用，但多机协作这块还没有探索，限制了他的应用范围

% 2. multi-agent VLA很难，因为训练在单机上面做，没法scale up. 现有方法只解决了SFT问题，没解决RFT问题

% 3. （一个很短的段落）我们做了个pipeline，包含三步，实现了multi-agent VLA的RFT

% 4. 展开三个方法。为了解决xxx问题，我们提出xxx方法，实现了xxx效果

Vision-language-action (VLA) models have rapidly emerged as a general interface between visual perception, language understanding, and low-level robotic control, mastering diverse single-agent tasks within a single policy \citep{ghosh2024octo,kim2024openvla,black2024pi0,physicalintelligence2025pi05}. However, real-world robot manipulation often require robots to collaborate with others, such as multi-arm assembly \citep{zhang2026dexora} and multi-robot household manipulation \cite{chen2025robotwin2}. Despite the advances in single-agent tasks, such performance do not transfer to multi-agent setting automatically, limiting the applicability of VLAs.

Extending VLAs to multi-agent setting is nontrivial. VLAs are typically pre-trained on web-scale single-agent data to acquire general-purpose manipulation skills. However, multi-agent cooperation requires fine-grained low-level coordination skills between robots, which is not covered in the single-agent pretraining dataset. Existing works in vision-language model (VLM) domain focus on decomposing multi-agent task to single-agent plans \citep{mandi2023roco,liu2024coherent,qin2025robofactory,tan2025roboos,chen2025robotwin2}. However, their low-level execution either relies on rule-based policy or assigned to existing VLAs, overlooking the fine-grained low-level action details required for VLA cooperation. Recently, CHORUS \citep{doshi2026chorus} demonstrate the feasibility of multi-VLA cooperation by supervised fine-tuning, yet its performance is inherently bounded by demonstration data and cannot further improve from its own experience.

In this paper, we present a three-stage recipe for reinforced fine-tuning of multi-agent VLAs, consisting of data collection, offline RL, and online RL, which together bridge the gap between single-agent and multi-agent VLAs. First, we propose initialization-aware data collection to ensure that the pretrained single-agent VLA possesses the basic skills required for multi-agent cooperation. The key insight is that subsequent RL training is only possible if the pretrained VLA achieves a non-zero success rate across different environment initializations. For initializations where the VLA succeeds at least occasionally, we collect both successful and failed trajectories for further tuning. For initializations where the VLA never succeeds, we conclude that the single-agent VLA lacks the primitive skills required for that configuration, and we invoke human demonstrations for these cases only, thereby inducing acceptable human effort. Together, this process yields an offline dataset with broad coverage of initial conditions with non-zero success rate, at reasonable human cost.

The second stage uses offline RL to reinforce behaviors from successful trajectories and penalize failure cases. A fundamental challenge here is credit assignment across agents~\citep{rashid2020monotonic}. In a multi-agent rollout, overall task failure does not imply that every agent behaved poorly, since one agent may have executed its subtask flawlessly while another caused the failure. Consequently, fine-tuning on entire joint rollouts may credit individual agents improperly, either penalizing desirable individual behavior in failed rollouts or rewarding flawed individual behavior in successful ones. To address this, we propose offline credit-filtered tuning, which estimates a per-agent advantage capturing the contribution of each agent to the team outcome, and fine-tunes the VLA only on per-agent trajectory with positive advantage, ensuring that only desirable behaviors are reinforced.

The third stage further improves the VLA via online RL. We find that existing online RL methods for single-agent flow-matching VLAs~\citep{chen2026pirl} can be unstable in hard multi-agent tasks with low success rates, a failure we attribute to noisy co-exploration and unstable concurrent updates. First, these methods explore by adding independent noise in raw action space of each agent, which disrupt precise coordination and collapse the learning signal. Second, the intractable marginal likelihood of flow matching forces high-variance surrogate gradients over the denoising chain. While tolerable for a single agent, this variance is amplified by the non-stationarity of concurrent multi-agent updates. To overcome these issues, we instead perform online RL in the latent noise space of the flow-matching policy, using DSRL~\citep{wagenmaker2025dsrl}. Rather than modifying VLA weights, we freeze them and induce diverse behaviors by steering the low-dimensional latent noise via online RL, so exploration remains confined to in-distribution behaviors, which yields stable training and further improves performance through trial and error in regime with low success rates.

% First, these methods explore by injecting stochasticity into the denoising process in raw action space; performed independently by multiple agents, such uncoordinated perturbations disrupt precise coordination and collapse the learning signal. Second, as the marginal action likelihood of flow matching is intractable, policy gradients must be computed through a surrogate sampler over the denoising chain, yielding high-variance updates, which is tolerable for a single agent, but amplified under multi-agent non-stationarity, where concurrent updates continually shift each agent's effective environment. 

Together, the three steps gives the first reinforced fine tuning recipe for multi-agent VLAs. 
% Empirically, our multi-agent VLA achieves end-to-end low-level control with xxx\% success on RoboTwin and xxx\% on RoboFactory, and these gains persist in real-world deployment on cooperative tasks with Franka robots.
We evaluate our framework with both $\pi_0$ and $\pi_{0.5}$ backbones across 11 tasks in RoboTwin, RoboFactory and real-world tasks with two Franka robots, covering varying agent numbers from 2 to 4. Our multi-agent VLA improves the average success rate by $+23.1\%$, $+16.4\%$, and $+44\%$ on RoboTwin, RoboFactory, and real-world tasks, respectively. Ablations show each stage is essential for success and visualize its effectiveness against alternative methods, exploring the potential of RL methods in multi-agent VLAs and critical factors that affects its success.

\textbf{Contributions.} Our contributions are two-folded. First, we propose multi-agent VLA, a three-stage reinforced fine-tuning pipeline for multi-VLA cooperation, spanning data collection, offline fine tuning and online fine tuning. Second, our multi-agent VLA achieves consistent enhancement of VLA collaborative across 11 tasks in RoboTwin, RoboFactory and real-world Franka robots.

\section{Related Work}
\label{sec-related-work}

\textbf{Vision-Language-Action Models.} VLA models provides a practical way of integraing visual perception, language understanding, and robotic control into a unified policy. Early systems such as RT-1~\citep{brohan2022rt1} and RT-2~\citep{brohan2023rt2} demonstrated large-scale multi-task learning and the transfer of web-scale semantic knowledge to physical actions, while Open X-Embodiment~\citep{oneill2023openx}, Octo~\citep{ghosh2024octo}, and OpenVLA~\citep{kim2024openvla} further advanced cross-embodiment pretraining and open generalist robot policies. Recent works increasingly adopt diffusion or flow-matching action decoders to model multimodal and temporally coherent action sequences, including Diffusion Policy~\citep{chi2023diffusionpolicy}, RDT-1B~\citep{liu2024rdt1b}, $\pi_0$~\citep{black2024pi0}, Diffusion-VLA~\citep{wen2024diffusionvla}, DexVLA~\citep{wen2025dexvla}, HybridVLA~\citep{liu2025hybridvla}, and DiT-A~\citep{hou2025dita}. Other models further explore efficient action tokenization~\citep{pertsch2025fast}, open-world generalization~\citep{physicalintelligence2025pi05}, spatial reasoning~\citep{lin2025onetwovla}, world modeling~\citep{cen2025worldvla, zhang2025dreamvla}, and lightweight deployment~\citep{shukor2025smolvla}. However, VLAs are pretrained on single-agent datasets, and lack the skills required for multi-agent cooperation, which motivates our paper.

\textbf{Reinforced Learning for VLAs.} % Reinforcement learning has recently been introduced to improve the capability of VLA policies beyond supervised imitation. Some methods preserve the pretrained policy and use value functions~\citep{nakamoto2024valueguidance}, residual policies~\citep{xu2024rldg}, or guidance mechanisms~\citep{yuan2024policydecorator} to refine its outputs, while others directly fine-tune VLAs through online reinforcement learning~\citep{guo2025openvlaorl}, consistency-policy-based reinforcement fine-tuning~\citep{chen2025conrft}, interactive post-training with sparse task rewards~\citep{tan2025riptvla}, scalable trajectory-level reinforcement learning~\citep{lu2025vlarl}, reward-driven reinforced fine-tuning~\citep{zhang2025reinbot}, temporal-feedback-based policy optimization~\citep{shu2025rftf}, or trajectory-wise group-relative policy optimization~\citep{chen2025tgrpo}. DSRL~\citep{wagenmaker2025dsrl} performs reinforcement learning in the initial-noise space of a frozen diffusion policy, providing stable and parameter-efficient exploration but remaining constrained by the capability of the fixed action decoder. In contrast, $\pi^{*}_{0.6}$~\citep{physicalintelligence2025pistar06} learns a value model from demonstrations, autonomous rollouts, and expert corrections, and uses binary advantage conditions to update the underlying VLA. In our paper, we demonstrate existing works cannot be applied straightforwardly to multi-agent settings, and provide the first RL pipeline for multi-agent VLA.
Reinforcement learning has recently been introduced to improve the capability of VLA policies beyond supervised imitation, including offline and online methods. Offline RL optimize policy from a fixed collection of demonstrations, typically by learning value functions and using them to guide policy optimization ~\citep{nakamoto2024valueguidance, chen2025conrft}. Human demonstrations are also included to correct VLA from mistakes ~\citep{physicalintelligence2025pistar06}. In contrast, online RL assumes continued interaction with environments and improve policy via trial-and-errors. The problem has been widely studied in VLAs with discrete action bins ~\citep{guo2025openvlaorl, li2025simplevlarl,lu2025vlarl,tan2025riptvla,zhang2025reinbot,shu2025rftf,chen2025tgrpo}. For flow matching and diffusion policy, the problem is more challenging since no tractable log-likelihood is available. \citep{xu2024rldg,luo2025serl,luo2024hilserl} explored this problem on small-model backbones, while later works extend it to $\pi$-series VLAs \citep{chen2026pirl,zhang2026reinflow} by approximating action log-likelihood. Besides, DSRL~\citep{wagenmaker2025dsrl} freezes the VLA policy and use RL to optimize the latent noise space, serving as an alternative to online RL methods. In our paper, we demonstrate existing works cannot be applied straightforwardly to multi-agent settings, and provide the first RL pipeline for multi-agent VLA.

\textbf{Embodied Multi-Agent Collaboration.} 
Embodied multi-agent collaboration has advanced rapidly with existing works broadly falling into two categories. The first focuses on high-level collaborative planning, where VLMs decompose a shared task into agent-specific subgoals and allocate them into different agents~\citep{mandi2023roco,kannan2023smartllm,zhang2023coela}, with further mechanisms such as execution feedback, embodiment awareness, and sequential evaluation to improve robust collaboration~\citep{liu2024coherent,chen2024emos,zhang2024read}.
The second line moves beyond symbolic planning toward action-level multi-robot execution. Benchmarks include RoboTwin~\citep{chen2025robotwin2} for bimanual manipulation together with rule-based low-level execution. RoboFactory~\citep{qin2025robofactory} extends the setting to cooperation with compositional constraints and more than two agents. Another line of work leverages VLMs for high-level task scheduling and coordination among multiple embodied agents~\citep{tan2025roboos,sun2025interactgen}, with subsequent works studying sim-to-real transfer~\citep{kang2026coenv}. However, these works rely on rule-based low-level control or simply assign the plan to VLAs, neglecting the fine-grained low-level skills required for multi-robot cooperation. The work most closely related to ours is CHORUS~\citep{doshi2026chorus}, which achieves multi-VLA cooperation via supervised fine-tuning. However, how to further improve the capability of multi-agent VLAs through RL remains largely unexplored.

\section{Preliminaries}
\label{sec:preliminaries}
\textbf{Problem Formulation.} Under the literature of Multi-Agent Reinforcement Learning (MARL) \citep{rashid2018qmix, yu2022surprising}, we formulate our problem as a Markov game \citep{littman1994markovgame}, defined by a tuple $\mathcal{G}=\langle \mathcal{N}, \Psi, \mathcal{S}, \{\mathcal{A}^i\}_{i=1}^N, \mathcal{P}, R, \gamma \rangle$. Here $\mathcal{N}=\{1, ..., N\}$ is the set of $N$ agents, $\Psi$ is space of language instructions, $\mathcal{S}$ is the state space, $\mathcal{A}^i$ is the individual action space, with $\mathcal{A}=\times_{i\in\mathcal{N}} \mathcal{A}^i$ the joint action space. $\mathcal{P}: \mathcal{S} \times \mathcal{A} \rightarrow \Delta(\mathcal{S})$ is the state transition probability. $R: \mathcal{S} \times \mathcal{A} \times \Psi \rightarrow \mathbb{R}$ is the shared reward function, $\gamma \in [0, 1)$ is the discount factor.

At time $t=0$, all agents receive a shared instruction $\psi \in \Psi$. At each timestep $t$, each agent observes $s_t \in \mathcal{S}$ and takes an action chunk $a^i_t \sim \pi^i(\cdot|s_t, \psi)$ drown from its policy, forming a joint action chunk $\mathbf a_t$. The joint policy is $\pi = \prod_{i=1}^N \pi^i$. After that, the environment proceeds to next state following the transition probability $\mathcal P(s_{t+1}|s_{t}, \mathbf{a}_t)$, and yields reward $r_t =  R(s_t, \mathbf a_t, \psi)$ for each agent. For generality, we avoid task-specific rewards and instead adopt a general reward scheme, including a penalty of $-100$ at the end of an episode if the task fails, and a step-wise penalty of $-1$ at each timestep to encourage faster task completion. Let $\rho_0$ be the initial state distribution, the joint objective is $J(\pi) = \mathbb E_{s_0 \sim \rho_0}[\mathbb E_{\pi}[\sum_{t=0}^\infty \gamma^t r_t|s_0, \psi]]$. The corresponding V and Q function is $V_\pi(s) = \mathbb E_{\pi}[\sum_{t=0}^\infty \gamma^t r_t|s_0=s, \psi]$ and $Q_\pi(s, \mathbf a) = \mathbb E_{\pi}[\sum_{t=0}^\infty \gamma^t r_t|s_0=s, \mathbf a_0 = \mathbf a, \psi]$, respectively.

\textbf{Handling Partial Observability.} In practice, each VLA may have local observation as input, rather than global state. The setting is called Decentralized Partially Observable Markov Decision Process (Dec-POMDP) \citep{oliehoek2016decPOMDP}. However, the process does not satisfy the Markov property when each agent only assess to partial observation, and is proven to be NEXP-complete \citep{bernstein2002complexity}. It is therefore standard in MARL to make theoretical analysis under full observability \citep{kuba2022trustregion, zhong2024heterogeneous, zhang2024read}. In practice, we feed camera RGB observations to the VLA backbone directly, following common MARL implementations that use partial observations as policy input \citep{rashid2018qmix,yu2022surprising}.

\textbf{Heterogeneous Agent Reinforcement Learning.} In multi-agent reinforcement learning, a fundamental challenge is to evaluate the contribution of each agent to the joint return, known as \emph{credit assignment} \citep{yang2020overview}. Heterogeneous agent reinforcement learning \citep{kuba2022trustregion,zhong2024heterogeneous} solves this by estimating the agent-wise advantage according to the \emph{multi-agent advantage decomposition} lemma. Agents are then updated sequentially with non-negative advantage, yielding a \emph{monotonic improvement guarantee} on the joint return at every on-policy update.

Formally, let $i_{1:m}$ denote an ordered subset $\{i_1,\ldots,i_m\} \in \mathcal{N}$, and let $-i_{1:m}$ be its complement. The multi-agent Q function is $Q_{\pi}^{i_{1:m}} \left(s, a^{i_{1:m}}\right) \triangleq \mathbb{E}_{\pi^{-i_{1:m}}}
\left[ Q_{\pi} \left( s, a^{i_{1:m}}, a^{-i_{1:m}} \right) \right]$. For disjoint sets $j_{1:k}$ and $i_{1:m}$, the multi-agent advantage function is $A_{\pi}^{i_{1:m}} \left(s, a^{j_{1:k}}, a^{i_{1:m}} \right) \triangleq
Q_{\pi}^{j_{1:k},i_{1:m}} \left(s, a^{j_{1:k}}, a^{i_{1:m}} \right) - Q_{\pi}^{j_{1:k}} \left(s, a^{j_{1:k}} \right)$. Then, we have the following multi-agent advantage decomposition lemma:
\vspace{-0.1in}
\begin{lemma}[\textmd{Multi-Agent Advantage Decomposition \citep{kuba2022trustregion}}]
In any cooperative Markov games, given a joint policy $\pi$, for any state $s$ and any agent subset $i_{1:m}$, the equation holds:
\begin{equation}
\label{eq:multi_agent_advantage_decomposition}
A_{\pi}^{i_{1:m}} \left( s, a^{i_{1:m}} \right) = \sum_{j=1}^{m} A_{\pi}^{i_j} \left( s, a^{i_{1:j-1}}, a^{i_j} \right).
\end{equation}
\end{lemma}
\vspace{-0.2in}

\section{Method}
\label{sec:method}
% To improve the task efficiency, we propose a framework that enables MA-VLA systems to learn form their own experience. We first train a sequential value function using offline interaction data. This formulation enables the contribution of each agent to be calculated accurately. Based on that, we develop a two-stage policy optimization procedure. In the offline stage, we derive agent-specific advantage, use it to filter the offline datasets with negative samples and finetune each VLA policy. In the online stage, we follow the latent-space RL learning paradigm of DSRL~\citep{wagenmaker2025dsrl}, while incorporating the decomposed advantage as an auxiliary learning signal.
% In this paper, we present a three-stage recipe for reinforcement fine-tuning of multi-agent VLAs, consisting of data collection, offline RL, and online RL. During data collection, when the VLAs cannot reliably accomplish the task with specified initialization, we introduce expert interventions to provide corrective demonstrations and bootstrap policy improvement. In the offline RL stage, we estimate agent-specific advantages and fine-tune each VLA only on positive-advantage trajectories. In the online RL stage, since direct policy-gradient updates may destabilize the pretrained policy, we adopt DSRL~\citep{wagenmaker2025dsrl} to perform online reinforcement learning in the latent noise space of the flow-matching policy.

In this section, we introduce a three-stage recipe for multi-agent reinforced fine tuning, including data collection for broad distribution coverage, offline fine-tuning to reinforce desired behaviors and online fine-tuning that tunes latent noise space in flow matching to improve from trials and errors. The overall process is illustrated in Fig.~\ref{fig:framework}. The pseudo code is available in Appendix~\ref{sec:pseudo_code}.
\begin{figure*}[!t]
\centering
\includegraphics[width=1\textwidth]{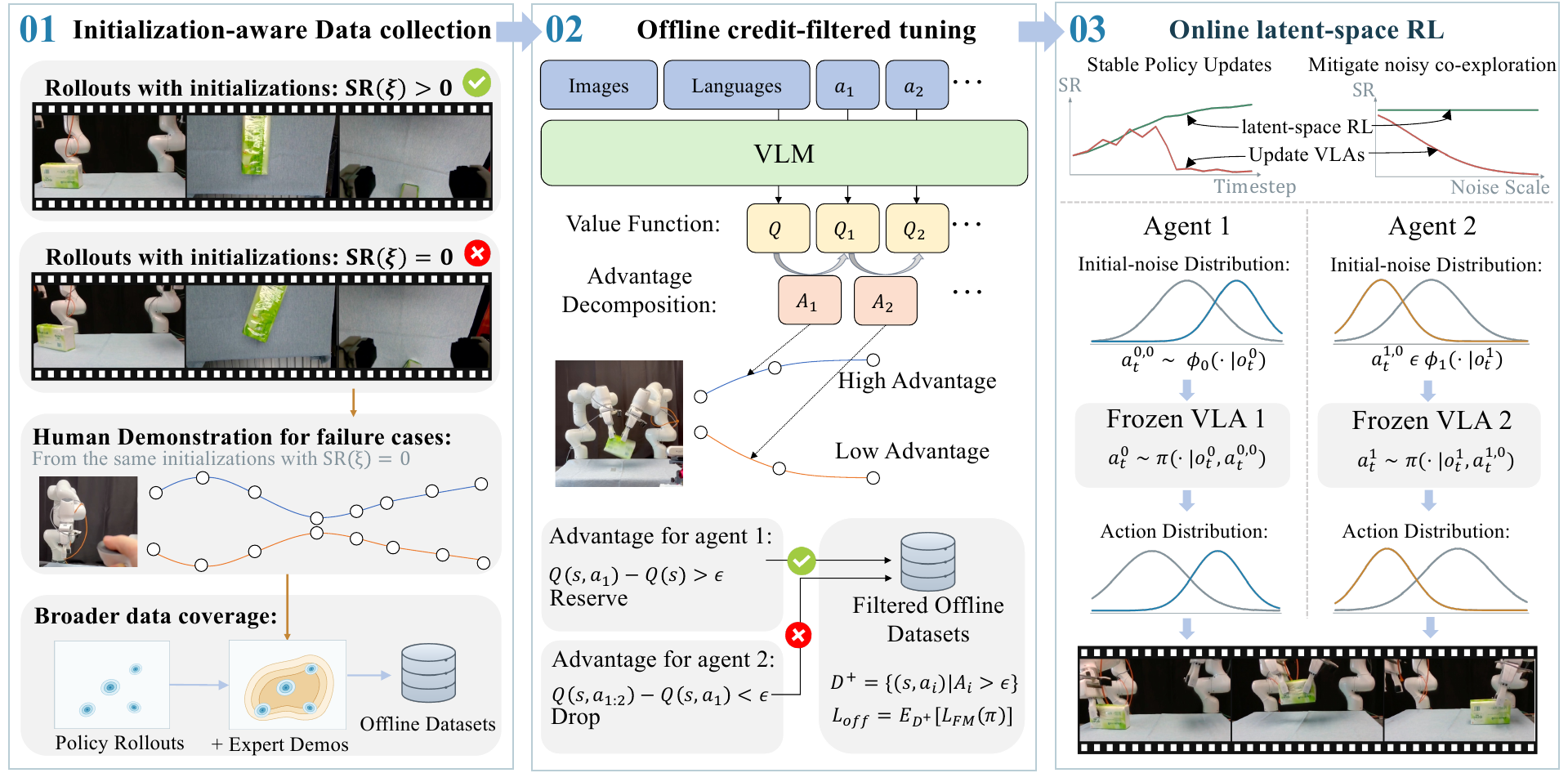}
\vspace{-0.2in}
\caption{The three-stage recipe for reinforcement fine-tuning multi-agent VLAs. We first perform \emph{Initialization-Aware Data Collection} to broaden coverage of environment conditions, followed by \emph{Offline Credit-Filtered Tuning} to decompose agent-wise advantages and reinforce high-contribution behaviors. Finally, we extend DSRL to the multi-agent setting and perform \emph{Online Latent-Space Reinforcement Learning} to enable further improvement through trial and error.}
\vspace{-0.2in}
\label{fig:framework}
\end{figure*}
% TODO: change the pictures in Stage 3.

\subsection{Initialization-Aware Data Collection}
\label{sec:init_aware_collection}

% A general-purposed VLA needs to success under diverse environment configurations, including object poses, robot poses, task layouts, and others. However, current works have demonstrated that existing VLAs are non-robust when these initial conditions varies \citep{wang2025vlatest}. Moreover, in challenging multi-agent settings, VLAs often achieve zero success rate in some initializations, receiving similar negative reward for all policy rollouts that makes reinforcement learning impossible. In such cases, we conclude current VLAs lack the primitive skills under this configuration, and propose an initialization-aware data collection scheme that calls for human demonstrations under such conditions.

A general-purpose VLA must succeed across diverse environment configurations, including object poses, robot poses, and task layouts. However, recent studies have shown that existing VLAs are brittle to variations in these initial conditions \citep{wang2025vlatest}. Moreover, in challenging multi-agent settings, VLAs can achieve a zero success rate under certain initializations. Every policy rollout then receives the same negative reward, rendering reinforcement learning infeasible. We attribute such failures to the lack of the primitive skills for these configurations, and propose an initialization-aware data collection scheme that calls human demonstrations only under such conditions.

% To handle this, let $\xi \in \Xi$ denote an environment initialization. For each initialization $\xi$, we execute the pretrained policy for $K$ episodes and estimate its empirical success rate (SR) as $SR(\xi)$. When $SR(\xi) > 0$, reinforcement learning is possible and we collect these $K$ policy rollouts containing both successful and failed cases for subsequent reinforced fine-tuning. In contrast, when $SR(\xi) = 0$, the pretrained VLA cannot succeed in current initialization. We interpret this as an indication that the current policy lacks the primitive behavior required for this initialization, such that sparse-reward exploration alone is unlikely to bootstrap effective learning within a reasonable interaction budget. We therefore calls for same $K$ human demonstrations for these initializations only, which provide a minimal set of successful behaviors that expands the support of the policy and enables subsequent reinforcement learning to operate in regions that would otherwise contain no positive signal.

Formally, let $\xi \in \Xi$ denote an environment initialization. For each $\xi$, we execute the pretrained policy for $K$ episodes and estimate its empirical success rate $\mathrm{SR}(\xi)$. When $\mathrm{SR}(\xi) > 0$, reinforcement learning is feasible, and we retain these $K$ rollouts containing both successful and failed trajectories for subsequent reinforced fine-tuning. In contrast, $\mathrm{SR}(\xi) = 0$ indicates that the pretrained policy lacks the primitive behavior required for this initialization, and that sparse-reward exploration alone is unlikely to bootstrap effective learning within a reasonable interaction budget. For these initializations only, we calls for human experts for successful demonstrations, giving a minimal set of successful behaviors that expands the support of the policy and enables reinforcement learning to operate in regions that would otherwise yield no positive signal.

Despite being a simple solution, our initialization-aware data collection gathers data with broad coverage of environment conditions, making subsequent reinforcement learning possible. It also strikes a balance between fully autonomous exploration and exhaustive human demonstration collection, invoking human effort only when necessary.

\subsection{Offline Credit-Filtered Tuning}
\label{sec:credit_filtered_tuning}

% Initialization-aware data collection stage provides a mixed offline dataset for training. Next, we use offline RL to reinforce desired behaviors. During the cooperation of multiple VLAs, a challenge in offline learning is the \emph{credit assignment} between agents \citep{rashid2018qmix}. Since the final outcome is a joint cooperation of all agents, it can be hard to evaluate the individual contribution of each agents in a rollout. To reinforce desired agent behavior only, we propose offline credit-filtered tu ning that estimates the contribution of each agent by its individual advantage, and only learn from transitions with positive advantage.

Given the mixed offline dataset produced by initialization-aware data collection, we apply offline RL to reinforce desired behaviors. A key obstacle in the multi-VLA setting is \emph{credit assignment} \citep{rashid2018qmix}. Since task success is a joint outcome of all agents, an individual agent's contribution to a rollout is hard to assess. We therefore propose offline credit-filtered tuning, which quantifies each agent's contribution by its individual advantage and updates the policy only on transitions with positive advantage, thereby reinforcing beneficial behaviors while filtering out the rest.

Formally, given an offline dataset $\mathcal{D}_{\mathrm{offline}}$, we estimate the agent-wise advantage as follows. Let $i_{1:N} = {i_1, i_2, ... i_N}$ an arbitrary fixed decision order of $N$ agents, with $a^i_{1:m} = (a^{i_1}, ..., a^{i_m})$ for the first $m$ actions, with $a^i_0 = \varnothing$. Given the discounted return $G_t = \sum_{t'=t}^{T} \gamma^{t'-t} r_{t'}$ of a trajectory in the dataset, we fit an agent-combined Q function by sequentially regressing onto the same return:
\begin{equation}
\label{eq:offline_value_loss}
\min_{Q} \mathbb E_{\tau \in \mathcal D_{\mathrm{offline}}} \sum_{m=0}^N \left[ H(G_t, Q_\pi^{i_{1:m}}(s_t, \textbf{a}_t^{i_{1:m}})) \right],
\end{equation}
where the two boundary cases are the V function, $Q_\pi^{i_{1:0}}(s_t) = V_\pi(s_t)$, and the joint value function $Q_\pi^{i_{1:N}}(s_t, \textbf{a}_t^{i_{1:N}}) = Q_\pi(s_t, \mathbf a_t)$. Every intermediate $Q_\pi^{i_{1:m}}$ corresponds to the value after revealing the action taken by first $m$ agents, before the rest $N-m$ agents have acted.

The critic is implemented as a transformer that accepts a variable‑length action input, so all $N+1$ terms are learned by a shared network in one pass. Following the distributional value function design of $\pi_{0.6}$ \citep{physicalintelligence2025pistar06}, which yields better empirical results, each Q function outputs a 201-bin distribution over returns, and $H$ is the cross entropy between empirical return $G_t$ and the predicted distribution. Given the learned agent-combined Q function, the advantage of each individual agent is obtained as a telescoping difference between consecutive orders:
\begin{equation}
\label{eq:agent_advantage}
A_{\pi}^{i_m}(s_t, \mathbf a^{i_{1:m-1}}_t, a^{i_m}_t) = Q_{\pi}^{i_{1:m}} \left( s_t, \mathbf a_t^{i_{1:m}} \right) - Q_{\pi}^{i_{1:m-1}} \left( s_t, \mathbf a_t^{i_{1:m-1}} \right), \quad m = 1, 2, ... N.
\end{equation}
By the multi-agent advantage decomposition lemma \citep{kuba2022trustregion}, the sum of individual advantage corresponds exactly to the joint advantage $A_{\pi}^{i_{1:N}} \left( s, a^{i_{1:m}} \right) = \sum_{j=1}^{N} A_{\pi}^{i_j} \left( s, a^{i_{1:j-1}}, a^{i_j} \right)$.

The estimated agent-wise advantage gives us a per-agent notion of desirable behavior. An action is worth reinforcing when it improves upon what that agent would have done on average under the behavioral policy $\pi$, \ie, when $A_\pi^{i_m}>0$. This results in an agent-wise advantage filtered dataset:
\begin{equation}
\label{eq:positive_dataset}
\mathcal{D}^{+} = \Big\{ \left( s_t, a_t^{i_m} \right) \in \mathcal{D}_{\mathrm{offline}} \big| \;A^{i_m}_{\pi}(s_t, \mathbf a^{i_{1:m-1}}_t, a^{i_m}_t)>0\Big\} , \quad m = 1, 2, ... N .
\end{equation}
Afterwards, for each individuals, we perform offline fine-tuning on $\mathcal D^+$  using the standard conditional flow-matching objective in $\pi_0$ \citep{black2024pi0}:
$\mathcal{L}_{\mathrm{offline}} = \mathbb{E}_{\mathcal{D}^{+}} \left[ \mathcal{L}_{\mathrm{FM}} \left( \mathbf{\pi} \right) \right]$.
\begin{proposition}[Monotonic Policy Improvement]
Let $\pi^+ \in \argmin_{\pi} \mathbb{E}_{\mathcal{D}^{+}}[\mathcal{L}_{\mathrm{FM}} \left( \mathbf{\pi} \right)]$. Assume (1) the critic gives accurate advantages $A^{i_m}_\pi$ on the states visited by $\pi^+$, and $\pi^+$ operates in the support of $\mathcal D^+$, (2) an action kept by the filter has non-negative expected advantage when the preceding agents $i_{1:m-1}$ act according to $\pi^+$, then we have $J(\pi^+) - J(\pi)\geq 0$.
\end{proposition}
\emph{Proof sketch.} By performance difference lemma \citep{kakade2002approximately} and multi-agent advantage decomposition lemma, the improvement over expected return is defined as the expectation over agent-wise advantage. Since actions keep a non-negative advantage with preceding agents follows $\pi^+$, the expected return cannot get worse. The proof details are in Appendix~\ref{sec:proof}.

\subsection{Online Latent-Space Reinforcement Learning}
\label{sec:online_latent_rl}

% While offline RL can outperform demonstrations by stitching suboptimal trajectories, it cannot query the environment to explore new behaviors outside the data support, which can be lifted via online RL. Note that directly applying online RL to flow-matching VLAs is challenging, since flow matching do not provide a tractable action likelihood and must be approximated \citep{zhang2026reinflow}. To our knowledge, $\pi_{RL}$ \citep{chen2026pirl} is the only method for flow-matching VLAs with open-source implementations.

Offline RL can exceed demonstration quality by stitching suboptimal trajectories, but it cannot interact with the environment and thus cannot discover behaviors beyond the data support. As such, we further tune multi-agent VLAs via online RL. Extending online RL to flow-matching VLAs is nontrivial, since flow matching yields no tractable action likelihood and the log-probability must instead be approximated. Consequently, existing works mainly operate at small policy scales \citep{zhang2026reinflow, xu2024rldg, luo2025serl}, with $\pi_{RL}$ \citep{chen2026pirl} the only VLA-scale research with publicly available implementation \citep{yu2026rlinf}.

However, while $\pi_{RL}$ works well for simple multi-agent tasks, the success do not extend to harder settings. We hypothesize the reason to noisy co-exploration and unstable policy updates. First, the exploratory noises in online RL fine-tuning consists of unstructured Gaussian noise injected to each agent independently. Such uncoordinated noise disrupts the cooperative behavior of the offline-tuned policies and can decrease the success rate below its pre-RL level. Second, the episode reward is sparse and distributed across both time horizon and denoising steps, yielding high-variance gradient updates. While this variance is tolerable in the single-agent setting, it is exacerbated in the multi-agent setting by non-stationarity. From agent $i$'s perspective, the policy changes of other agents are absorbed in environment dynamics, violating the stationarity assumption required by RL.

Guided by this diagnosis, we turn to online RL in the latent noise space of the flow-matching policy, and adopt DSRL \citep{wagenmaker2025dsrl} as our base algorithm. Formally, consider a policy $\pi(a|s)$, with agent index dropped for simplicity. $\pi$-series VLAs employ rectified flow \citep{liu2022rectified}, which is a linear ODE from noise $a^0 \sim \mathcal N(0, \mathbf I)$ to the final action output $a^1$ during inference, with a linear interpolation $a^\tau_t = \tau a^1_t + (1-\tau) a_t^0$, $\tau \in [0,1]$. Instead of tuning the rectified flow, DSRL tunes the latent noise $a^0$. In place of sampling from Gaussian distribution $\mathcal N(0, \mathbf I)$, they learn the latent noise from a policy $\phi(a^0|s)$ using online RL. This paradigm avoids tuning the weights of VLA and provides a lightweight solution since $\phi(a^0|s)$ can be a small network instead of large VLAs.

DSRL serves as a clean remedy against the problem of noisy co-exploration and unstable policy updates in multi-VLA. First, since the VLA is fixed, exploration is confined to the latent noise space, producing in-distribution, temporally coherent behaviors. Joint exploration thus preserves basic cooperation from offline fine tuning and results in non-degenerate success rate. Second, tuning in latent space avoids modeling the intractable marginal likelihood of flow matching, and constrains policies within the set of offline-tuned behaviors, lowering policy variance during update and thus bounding the non-stationarity during training.

\section{Experiments}
\textbf{Experimental setting.}
We evaluate our multi-agent VLA on a total of 11 cooperation tasks. To evaluate closely synchronized bimanual actions, we consider four tasks in RoboTwin \citep{chen2025robotwin2}, including \textit{Lift Pot}, \textit{Grab Roller}, \textit{Handover Mic}, and \textit{Put Bottles Dustbin}. To evaluate coordination of three or four agents, we consider four tasks in RoboFactory \citep{qin2025robofactory}, including \textit{Camera Alignment}, \textit{Take Photo}, \textit{Three Robots Stack Cubes}, and \textit{Long Pipeline Delivery}. 
We train low-level end-to-end VLAs for these tasks.
% \emph{Note that these environments use rule-based method for low-level control, while our method replace the low-level rule to end-to-end VLAs}. 
For real-world evaluation, we design three tasks that requires cooperation of two Franka robots, including \textit{Handover}, \textit{Lift}, and \textit{Handover \& Drop}. 

\textbf{Baselines.} Due to the scarcity of researches in RL methods for multi-agent VLAs, we use CHORUS~\citep{doshi2026chorus} as a baseline for supervised fine-tuning, and implement our data collection, offline RL and online RL on its checkpoint. For online stage, we additionally compare our choice of DSRL with Flow-SDE and Flow-Noise~\citep{chen2026pirl}, with results reported in ablations. Our backbones include $\pi_0$ \citep{black2024pi0} and $\pi_{0.5}$ \citep{physicalintelligence2025pi05}.

\textbf{Implementation Details.}
We adopt a decentralized execution setting similar to CHORUS~\citep{doshi2026chorus}. Each VLA receives its local observations including wrist-camera RGB images and proprioceptive states. All agents additionally share a third-person RGB observation and the language instruction. % We do not share parameters between VLAs since the task require different individual skills and our pilot study do not find consistent performance gains. 
For each simulation task, we collect 50 expert demonstrations for supervised fine-tuning and 30 demonstrations for real-world tasks. During the initialization-aware data collection stage, we evaluate 96 environment initialization configurations for each simulation task and 50 configurations for each real-world task, with $K=5$ the number of rollouts for each configuration in simulation and $K=1$ for real world. We evaluate each task using 100 rollouts in simulation and 50 rollouts in real-world. See additional implementation details in Appendix~\ref{sec:implementation_details}.

\subsection{Simulation Results}
We evaluate our multi-agent VLA on 8 simulation tasks in two environments. We include our initialization-aware data collection automatically and report our offline and online RL stages as \emph{+ Offline RL} and \emph{+ Online RL}, respectively. As shown in Fig. \ref{fig:sim_results}, considering learning stages, our method improves an average of $+8.1\%$ after offline RL and $+19.7\%$ after online RL. As for backbones, our method improves an average of  $+18.9\%$ on $\pi_0$ and $+20.6\%$ on $\pi_{0.5}$. As for environments, our method improves an average of $+23.1\%$ on bimanual RoboTwin tasks and $+16.4\%$ on RoboFactory with 3-4 agents. Together, the results show the effectiveness of our multi-agent VLA on both online and offline RL stages, different backbones, and varying agent numbers.

% We report the task success rates across all eight tasks in Fig.~\ref{fig:sim_results}. Our method achieves consistent performance improvements across different stages of the training pipeline. Specifically, CHORUS achieves average success rates of $42.6\%$ and $73.8\%$ on RoboTwin and RoboFactory, respectively. After \emph{Offline Credit-Filtered Tuning}, the average success rates increase to $51.4\%$ and $80.0\%$, and further improve to $65.8\%$ and $87.5\%$ after \emph{Online Latent-Space Reinforcement Learning}. Overall, our complete post-training pipeline improves the average success rate over the CHORUS baseline by $23.1\%$ and $13.8\%$ percentage points on RoboTwin and RoboFactory, respectively, demonstrating the effectiveness of our multi-stage training recipe.

% We evaluate our multi-agent VLA post-training framework on eight simulation tasks across two benchmarks and two VLA backbones. We first use CHORUS~\citep{doshi2026chorus} as the supervised fine-tuning baseline, trained on the collected expert demonstrations. Starting from the CHORUS models, we perform \emph{Initialization-Aware Data Collection} and subsequently apply \emph{Offline Credit-Filtered Tuning} ("+ Offline RFT" in Fig.~\ref{fig:sim_results}). The advantage threshold is chosen such that approximately 5\% of the original samples are retained after filtering, while human demonstrations are assigned an advantage of 1. Finally, we perform \emph{Online Latent-Space Reinforcement Learning} without CTDE or parameter sharing ("+ Offline RFT + Online RL" in Fig.~\ref{fig:sim_results}).

\begin{figure*}[t]
\centering
\includegraphics[width=1\textwidth]{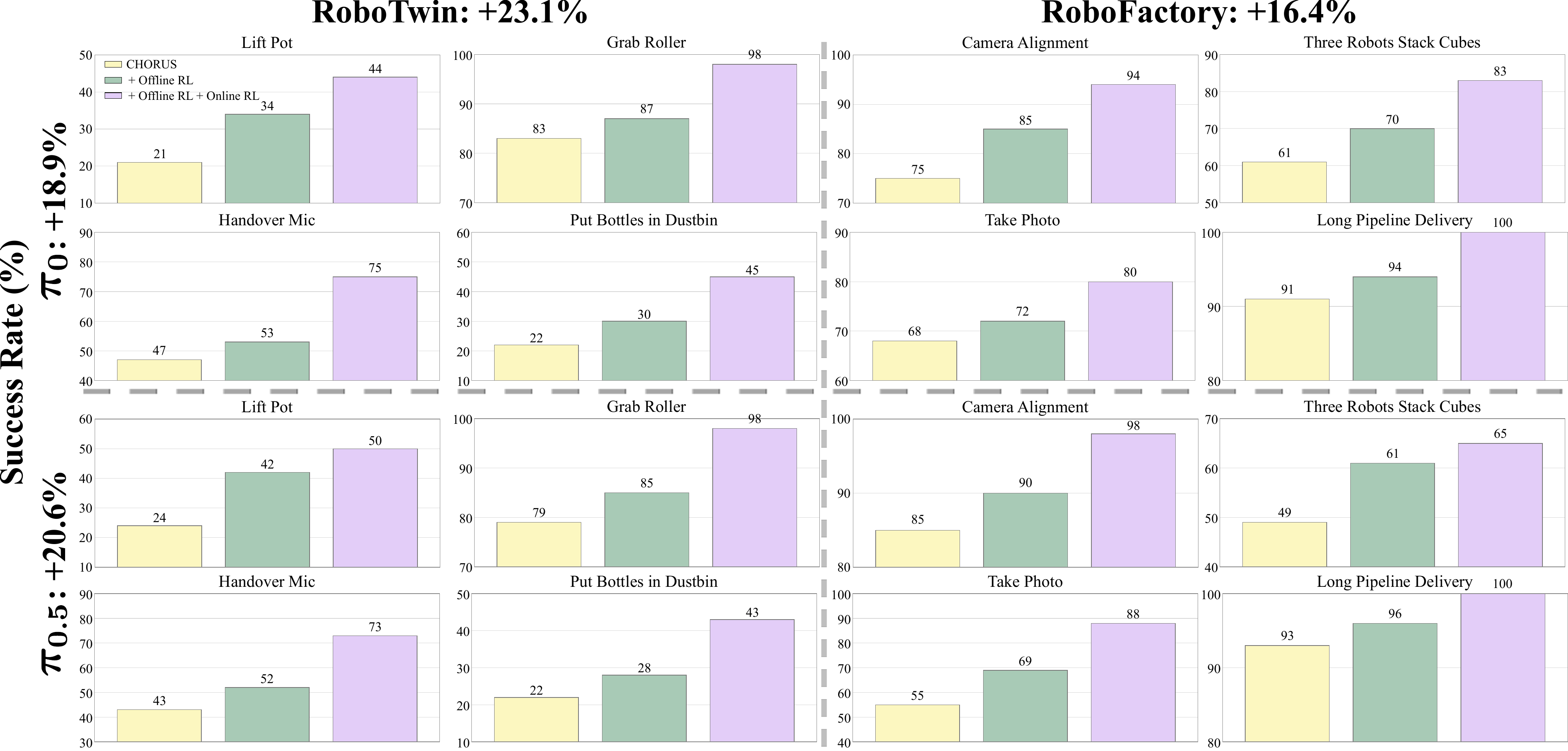}
\vspace{-0.2in}
\caption{Success rate (\%, $\uparrow$) under 8 different simulation tasks from RoboTwin and RoboFactory with backbone models $\pi_0$ and $\pi_{0.5}$.} % TODO: Legend. Pi05 Take Photo
\vspace{-0.1in}
\label{fig:sim_results}
\end{figure*}

\subsection{Ablations}

In this section, we provide extensive ablation on each state of our Multi-Agent VLA with visualization of the effectiveness of each parts. We select $\pi_{0.5}$ as backbone and consider four harder tasks that can not reliably solved by current methods to provide insights on potential on further gains.

% To examine whether the learned advantage provides meaningful agent-wise credit assignment, we visualize the predicted advantages for representative trajectories of the Lift Pot task. This task requires two robot arms to jointly lift a pot while maintaining similar gripper heights, making it particularly suitable for analyzing the contribution of each agent.

\begin{wraptable}{r}{0.60\textwidth}
    \centering
    \vspace{-0.2in}
    \caption{
        Ablation study on the initialization-aware data collection.
        We report success rate (\%, $\uparrow$).
    }
    \label{tab:ablatio_ia}
    \resizebox{\linewidth}{!}{
    \begin{tabular}{lccccc}
        \toprule
        \textbf{Task} & \textbf{CHORUS} & \makecell{\textbf{+ Offline RL} \\ \textbf{w/o IA}} & \makecell{\textbf{+Offline RL} \\ \textbf{+Online RL w/o IA}} & \textbf{+Offline RL} & \makecell{\textbf{+Offline RL} \\ \textbf{+Online RL}} \\
        \midrule
        \textit{Lift Pot} & 24 & 30 & 46 & 42 & \textbf{50} \\
        \textit{Handover Mic} & 43 & 46 & 63 & 52 & \textbf{73} \\
        \textit{Three Robots Stack Cubes}& 49 & 57 & 58 & 61 & \textbf{65} \\
        \textit{Take Photo} & 55 & 68 & 79 & 69 & \textbf{88} \\
        \midrule
        \textbf{Average} & 42.8 & 50.3 & 61.5 & 56.0 & \textbf{69.0} \\
        \bottomrule
    \end{tabular}
    }
    \vspace{-0.1in}
\end{wraptable}
\textbf{Ablation on initialization-aware data collection.} We show the effectiveness of our initialization via offline RL on alternative data without human demonstrations (\emph{+offline RL w/o init}). As shown in Table~\ref{tab:ablatio_ia}, adding human demonstrations provides better coverage, resulting in +5.7\% success rate on average. Additionally, we show initialization benefits further RL. In Fig.~\ref{fig:data_coverage}, we report the success rate on 25 randomly sampled initializations. Our data collection results in broader success rate on different initializations via offline RL, and further enables broader success coverage in subsequent online RL. The result states the importance of data coverage in multi-agent VLAs, and also points out the current limitation of VLAs in generalizing against various initializations and conditions.

% To evaluate the effectiveness of this component, we perform offline credit-filtered tuning on the CHORUS models without introducing additional human demonstrations to expand data coverage, denoted as '+RFT w/o IA'. As shown in Table.~\ref{tab:ablatio_ia}, our data collection scheme provides better coverage for all tasks, resulting in +5.7\% success rate on average. 
% Additionally, we examine the success rate of each method across 25 environment initializations, set up with the same random seeds for all methods. As shown in Fig.~\ref{fig:data_coverage}, initialization-aware data collection provides successful behaviors for challenging configurations, thereby enabling subsequent reinforcement learning. For example, among the five initializations shown in the first row, CHORUS achieves zero success under all of them. Even after extensive offline tuning without initialization-aware data collection (``+Offline RFT w/o IA''), the success rates remain low, as the limited autonomous exploration rarely discovers successful behaviors under these configurations. In contrast, by selectively introducing additional human demonstrations for these zero-success initializations, our data collection scheme enables the policy to acquire successful behaviors in three of the five previously unsolved configurations.

\begin{figure*}[t]
\centering
\includegraphics[width=1\textwidth]{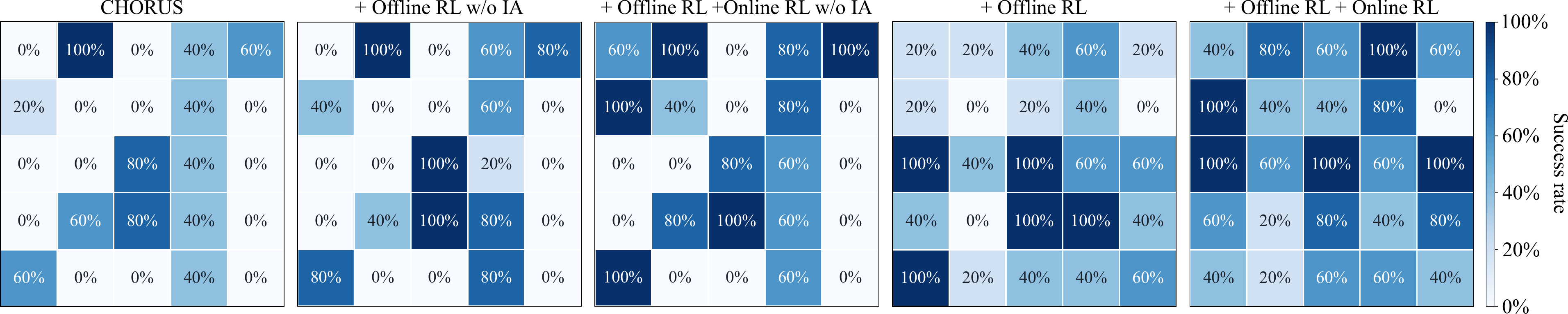}
\vspace{-0.2in}
\caption{Visualization of success rates under different environment initializations. \emph{Initialization-Aware Data Collection} identifies and compensates for challenging configurations in which agents struggle to acquire successful behaviors through autonomous exploration.}
\vspace{-0.1in}
\label{fig:data_coverage}
\end{figure*}

\begin{figure*}[t]
\centering
\includegraphics[width=1\textwidth]{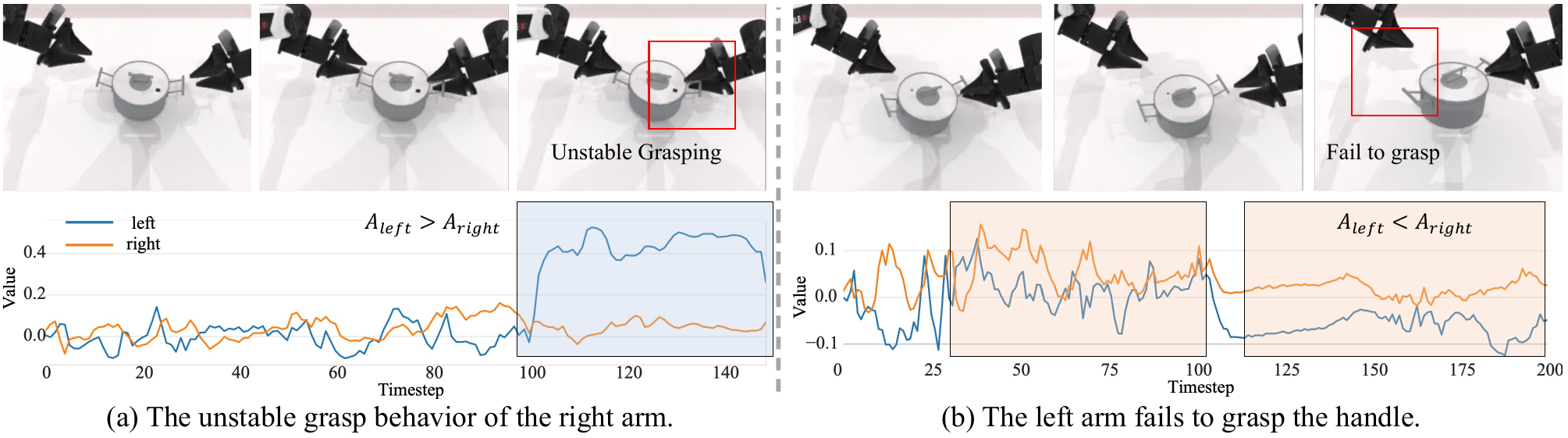}
\vspace{-0.2in}
\caption{Visualization of agent-wise credit assignments and the behaviors of them.}
\vspace{-0.2in}
\label{fig:visualization}
\end{figure*}

\textbf{Ablation on offline credit-filtered tuning.} As shown in Fig.~\ref{fig:sim_results}, offline RL provides an average +8.1\% success rate. We also visualize the agent-wise credit assignment in Fig.~\ref{fig:visualization}. In case (a), the advantage of right arm is lower than the left arm, corresponding to the unstable grasp behavior of the right arm. In case (b), the advantage of left arm is lower than the right arm, and in the video the left arm fails to grasp the handle, leading to task failure, showing the capability to attribute failure to each agents.

% . The visualization shows our method are able to attribute failure to different individuals.

% As shown in Fig.~\ref{fig:visualization}(a), when the task succeeds, both arms maintain similar gripper heights throughout most of the trajectory, and their advantages are predominantly positive. Toward the end of the trajectory, the right arm becomes less stable in grasping the pot, which is reflected by its lower advantage compared with the left arm. In Fig.~\ref{fig:visualization}(b), the right arm fails to securely grasp the handle, leading to task failure. Accordingly, its advantage remains substantially lower than that of the left arm. Importantly, the left arm still receives mostly positive advantages despite the failed episode, indicating that our critic correctly preserves its positive contribution rather than assigning the team-level failure to both agents. Fig.~\ref{fig:visualization}(c) exhibits the opposite case: the left arm fails to grasp the handle and receives a low advantage, whereas the right arm maintains a higher and predominantly positive advantage.

\begin{wraptable}{r}{0.5\textwidth}
    \centering
    \vspace{-0.2in}
    \caption{
        Comparison of different online RL methods.
        We report success rate (\%, $\uparrow$).
    }
    \label{tab:online_rl_comparison}
    \resizebox{\linewidth}{!}{
    \begin{tabular}{lcccc}
        \toprule
        \textbf{Task}
        & \textbf{Offline RL}
        & \textbf{+Flow-SDE}
        & \textbf{+Flow-Noise}
        & \textbf{+DSRL} \\
        \midrule

        \textit{Lift Pot}
        & 42 & 43 & 43 & \textbf{50} \\

        \textit{Handover Mic}
        & 52 & 56 & 62 & \textbf{73} \\

        \textit{Three Robots Stack Cubes}
        & 61 & 62 & 63 & \textbf{65} \\

        \textit{Take Photo}
        & 69 & 82 & 78 & \textbf{88} \\

        \midrule
        \textbf{Average}
        & 56.0 & 60.1 & 61.5 & \textbf{69.0} \\

        \bottomrule
    \end{tabular}
    }
\end{wraptable}
\textbf{Ablation on online latent-space reinforcement learning.} As shown in Fig.~\ref{fig:sim_results}, online RL is generally effective for multi-agent VLAs, improving +11.6\% on the offline tuned checkpoint. 
%Among online methods, DSRL works the best, surpassing Flow-Noise by xxx\% and Flow-SDE by xxx\% on average. We illustrate the lower success rate of flow-based methods in Fig. xxx. 
We further compare DSRL with Flow-Noise and Flow-SDE \citep{chen2026pirl} as alternative online RL methods. As shown in Table.~\ref{tab:online_rl_comparison}, DSRL achieves the highest success rate, outperforming Flow-Noise by +7.5\% and Flow-SDE by 8.9\% on average. To understand why Flow-based RL offers limited performance, we offer evidences to support our claim that RL methods for single-agent VLA cannot be transferred directly to multi-agent VLAs due to noisy co-exploration and unstable policy updates (Section \ref{sec:online_latent_rl}). To illustrate the effect of noisy co-exploration, we increase the exploratory noise of flow-based methods. As shown in Fig.~\ref{fig:noise_scale}, increasing the noise scale in Flow-SDE and Flow-Noise consistently reduces the success rate as policy explores. In contrast, DSRL operates in latent noise space initialized from $\mathcal{N}(0,I)$, which is identical to the original VLA setting and do not lose success rates. To illustrate the effect of unstable gradient updates, we illustrate learning dynamics of DSRL and flow-based methods in Fig. \ref{fig:training_curve}. We find flow-based noise can occasionally crash during training, while our DSRL improves performance in a stable way. Additionally, these experiments also calls for a multi-agent version of flow-matching algorithm for VLAs that handle the problem of noisy exploration and unstable gradients.

\begin{figure*}[t]
    \centering
    \begin{minipage}[t]{0.48\textwidth}
        \centering
        \includegraphics[width=\linewidth]{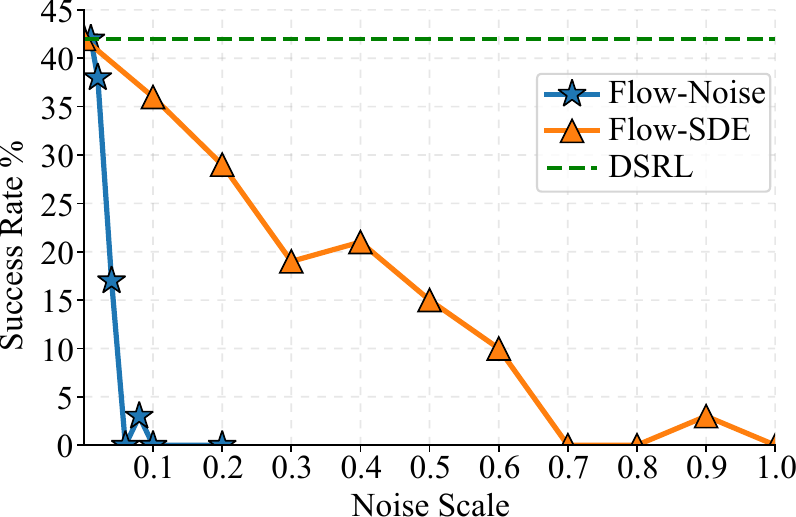}
        \vspace{-0.2in}
        \caption{The success rate under different noise scales of method Flow-Noise and Flow-SDE. Flow-based RL shows noisy co-exploration.}
        \label{fig:noise_scale}
    \end{minipage}
    \hfill
    \begin{minipage}[t]{0.48\textwidth}
        \centering
        \includegraphics[width=\linewidth]{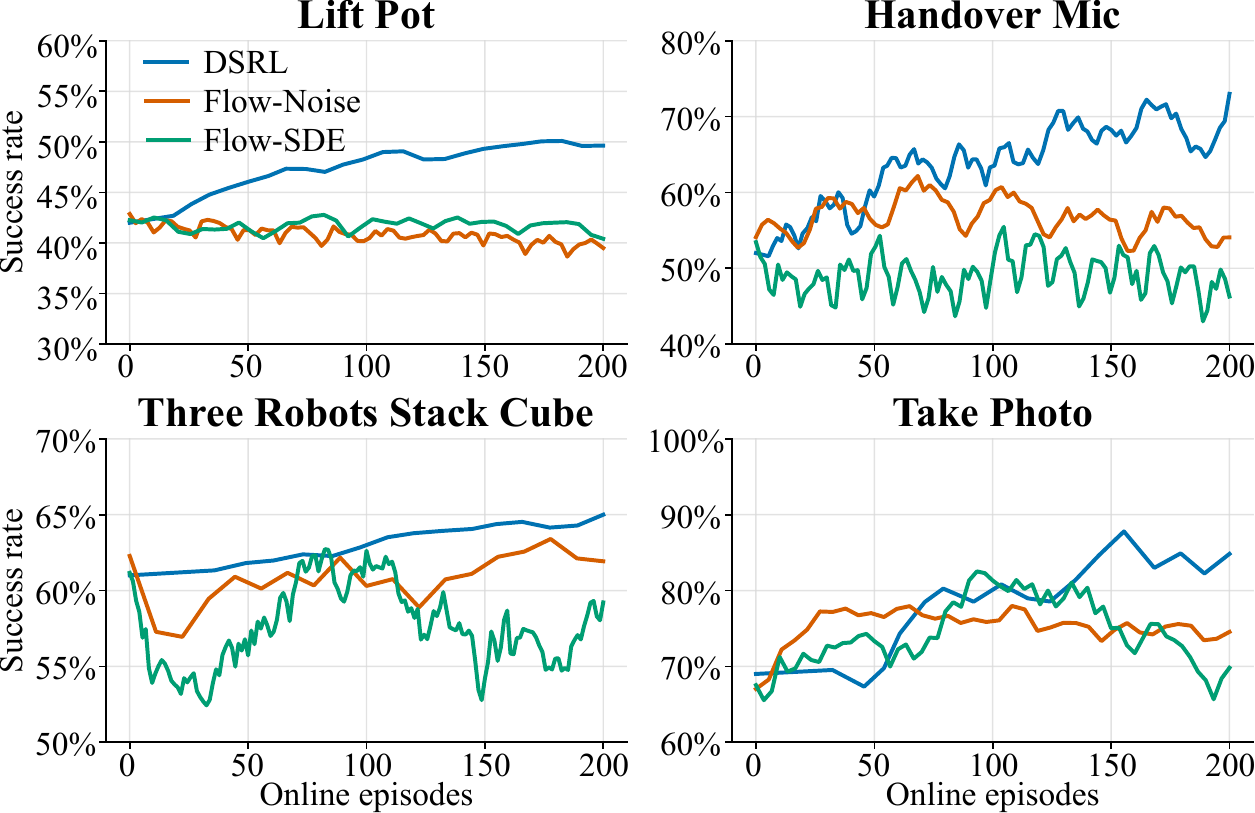}
        \vspace{-0.2in}
        \caption{Learning dynamics of online RLs. Flow-based RL shows unstable policy update.}
        \label{fig:training_curve}
    \end{minipage}
\end{figure*}

\subsection{Real-world Results}

Finally, we evaluate our Multi-Agent VLA in real-world settings. We consider $\pi_{0.5}$ backbone with three manipulation tasks that requires coordinating two Franka robots. In \emph{Handover}, one robot hands an object over to the other. In \emph{Lift}, two robots jointly grasp the handles of a basket and lift it. In \emph{Handover \& Drop}, one robot picks up a piece of trash and hands it to the other robot, which subsequently places it into a dustbin. We use the same three-stages pipeline in simulation, but change the online RL stage to DSRL with 50 real-world rollouts as a proof-of-concept result.

% See implementation details in Appendix~\ref{}.

% We use CHORUS~\citep{doshi2026chorus} as the supervised fine-tuning baseline, upon which we sequentially apply \emph{Offline Credit-Filtered Tuning} with \emph{Initialization-Aware Data Collection}, followed by \emph{Online Latent-Space Reinforcement Learning}. Implementation details are provided in Appendix~\ref{}.

\begin{figure*}[t]
\centering
\includegraphics[width=1\textwidth]{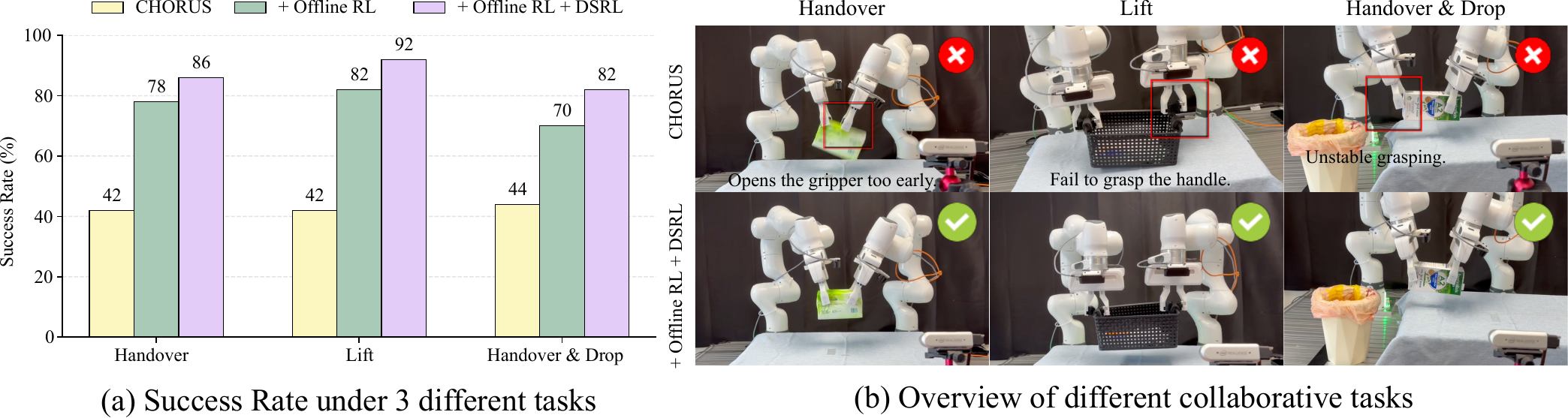}
\vspace{-0.2in}
\caption{The success rate (\%, $\uparrow$) and the overview under 3 different collaborative tasks on 2 Franka robots. Our training recipe effectively enhances the collaborative performance of multi-agent VLAs.} % TODO: Legend.
\vspace{-0.2in}
\label{fig:real_results}
\end{figure*}

The success rates can be seen in Fig.~\ref{fig:real_results}(a). Our multi-agent VLA stays consistently effective in real world, resulting in +34.0\% success rate by offline RL and +10.0\% by DSRL, with online RL improvements achieved by 50 real world rollouts only. We further provide a case study of our method and CHORUS in Fig.~\ref{fig:real_results}(b). The failure cases of CHORUS shows supervised fine tuning paradigm alone is insufficient to learn fine-grained VLA coordination. For \emph{Handover}, the left robot opens its gripper before the right robot securely grasps the object, causing it to fall onto the table. For \emph{Lift}, the right robot fails to securely grasp the basket handle before lifting. For \emph{Handover \& Drop}, the receiving robot fails to firmly grasp the trash handed over by the other robot, resulting in an unsuccessful handover. In contrast, our multi-agent VLA is free of these problems, exceeding the success rate of behavior cloning pipeline by large margins via reinforcing desired behaviors.

% improving average success rate xx\% by offline  RL 

% The results of the success rate can be seen in Fig.~\ref{fig:real_results}(a). After deploying offline finetuning, the success rate of three tasks yields an absolute improvement of x\%, y\% and z\% perspectively. Besides, after deploying online finetuning, the success rate yields furthermore improvement of x1\%, y1\% and z1\%, showing the effectiveness of our post-training pipeline.
% Fig.~\ref{fig:real_results}(b) further provides qualitative comparisons, where the first row shows representative episodes collected with CHORUS. For \emph{Handover}, the left robot opens its gripper before the right robot securely grasps the object, causing it to fall onto the table. For \emph{Lift}, the right robot fails to securely grasp the basket handle before lifting. For \emph{Handover \& Drop}, the receiving robot fails to firmly grasp the trash handed over by the other robot, resulting in an unsuccessful handover. In contrast, after applying our post-training recipe, the two Franka robots successfully coordinate to complete all three tasks, as shown in the second row.

\section{Conclusion}
In this work, we propose a three-stage reinforced fine-tuning pipeline for cooperative multi-agent vision-language-action models. First, initialization-aware data provide better distribution coverage by sweeping over initial configurations and invokes human demonstrations with zero success rates. Second, offline credit-filtered tuning rewards good desired individual behavior by VLA-level credit assignment, which assigns credit to individual agents and fine-tunes on per-agent trajectories with positive advantage rather than on entire joint rollouts. Third, online latent-space fine tuning freeze the VLA and perform RL in its latent noise space to reduce the challenge of noisy co-exploration and unstable gradient updates in online methods for VLAs. Experiments with $\pi_0$ and $\pi_{0.5}$ across 11 simulated and real-world tasks demonstrate consistent improvements, achieving gains of $23.1\%$, $16.4\%$, and $44\%$ on RoboTwin, RoboFactory, and real-world tasks, respectively.

\section*{AI use statement}
LLMs were employed for text polishing. The authors have thoroughly reviewed and validated all content presented in the paper.

\section*{Ethics statement}
Our work studies cooperation among multiple VLAs. The focus of this paper is to strengthen the ability of robot foundation models to coordinate with one another, thereby broadening their applicability to tasks such as collaborative assembly and multi-robot household assistance. To the best of our knowledge, our work raises no immediate ethical concerns. However, same as any works in robotics, deployment in real-world settings should be accompanied with appropriate safety safeguards and human oversight.

\section*{Reproducibility statement}
We provide implementation details in Appendix. \ref{sec:implementation_details}. Code is available at \url{https://anonymous.4open.science/r/mavla_rft-2BC0/}.

\bibliography{iclr2027_conference}
\bibliographystyle{iclr2027_conference}

\newpage
\appendix
{\LARGE\sc {Appendix for "Cooperative Multi-Agent Vision-Language-Action Models via Reinforced Fine Tuning"}\par}

\section{Pseudo code for our post-training pipeline.}
\label{sec:pseudo_code}

Algorithm~\ref{alg:multi_agent_vla} summarizes our three-stage reinforced fine-tuning pipeline, consisting of initialization-aware data collection, offline credit-filtered tuning, and online latent-space reinforcement learning.

\begin{algorithm}
\caption{Reinforced Fine-Tuning for Multi-Agent VLAs}
\label{alg:multi_agent_vla}
\begin{algorithmic}[1]

\Require Pretrained multi-agent VLA policies $\boldsymbol{\pi}$; environment initializations $\Xi$; number of rollouts $K$; advantage threshold $\epsilon$
latent policies $\boldsymbol{\phi}=\{\phi^i\}_{i=1}^{N}$
\Ensure Offline-tuned VLA policies $\boldsymbol{\pi}^{+}$ and latent-space policies $\boldsymbol{\phi=\{\phi^i\}_{i=1}^{N}}$

\State $\mathcal{D}_{\mathrm{offline}} \gets \emptyset$

\Statex
\Statex \textbf{Stage I: Initialization-Aware Data Collection}
\For{each initialization $\xi \in \Xi$}
    \State Collect $K$ rollouts $\mathcal{T}_{\xi} =\{\tau_{\xi}^{1},\ldots,\tau_{\xi}^{K}\}$ using $\boldsymbol{\pi}$
    \State Estimate empirical success rate $\mathrm{SR}(\xi)$
    \If{$\mathrm{SR}(\xi) > 0$}
        \State $\mathcal{D}_{\mathrm{offline}} \gets \mathcal{D}_{\mathrm{offline}} \cup \mathcal{T}_{\xi}$
    \Else
        \State Collect human demonstrations $\mathcal{D}_{\mathrm{human}}^{\xi}$ under initialization $\xi$
        \State $\mathcal{D}_{\mathrm{offline}} \gets \mathcal{D}_{\mathrm{offline}} \cup \mathcal{D}_{\mathrm{human}}^{\xi}$
    \EndIf
\EndFor

\Statex
\Statex \textbf{Stage II: Offline Credit-Filtered Tuning}
\State Compute discounted returns $G_t$
for trajectories in $\mathcal{D}_{\mathrm{offline}}$
\State Train the agent-combined critic
$Q_{\pi}^{i_{1:m}}$
for $m=0,\ldots,N$ using Eq.~\ref{eq:offline_value_loss}
\State $\mathcal{D}^{+} \gets \emptyset$
\For{each transition $(s_t,\mathbf{a}_t)$ in $\mathcal{D}_{\mathrm{offline}}$}
    \For{$m=1,\ldots,N$}
        \State Compute agent-wise advantage using Eq.~\ref{eq:agent_advantage}
        \If{$A_{\pi}^{i_m} > \epsilon$}
            \State $\mathcal{D}^{+} \gets \mathcal{D}^{+} \cup \{(s_t,a_t^{i_m})\}$
        \EndIf
    \EndFor
\EndFor
\State Fine-tune $\boldsymbol{\pi}$ on $\mathcal{D}^{+}$ with the flow-matching objective: $\boldsymbol{\pi}^{+} \gets \arg\min_{\boldsymbol{\pi}} \mathbb{E}_{\mathcal{D}^{+}} [\mathcal{L}_{\mathrm{FM}}(\boldsymbol{\pi})]$

\Statex
\Statex \textbf{Stage III: Online Latent-Space Reinforcement Learning}
\State Freeze the offline-tuned VLA policies
$\boldsymbol{\pi}^{+}$
\While{online RL has not converged}
    \State Reset the environment and observe
    $\{o_0^i\}_{i=1}^{N}$
    \For{$t=0,\ldots,T-1$}
        \For{each agent $i \in \mathcal{N}$}
            \State Sample latent noise $a_{t}^{i,0} \sim \phi^i(\cdot \mid o_t^i)$
            \State Generate action chunk using the frozen VLA: $a_t^i \sim \pi^{+,i}(\cdot \mid o_t^i,a_t^{i,0})$
        \EndFor

        \State Execute joint action $\mathbf{a}_t=(a_t^1,\ldots,a_t^N)$ and observe $r_t$ and next observations
        \State Store $(\mathbf{o}_t, \mathbf{a}_t^0, r_t, \mathbf{o}_{t+1})$ in the online replay buffer
    \EndFor
    \State Update $\boldsymbol{\phi}=\{\phi^i\}_{i=1}^{N}$ with DSRL using collected online experience
\EndWhile

\State \Return $\boldsymbol{\pi}^{+},\boldsymbol{\phi}$
\end{algorithmic}
\end{algorithm}

\section{Proof of Proposition 4.1}
\label{sec:proof}
We provide a theoretical justification for the monotonic improvement of
our offline credit-filtered tuning. Let
$\pi(\mathbf a\mid s)=\prod_{m=1}^{N}\pi^{i_m}(a^{i_m}\mid s)$
denote the behavioral policy that generates
$\mathcal D_{\mathrm{offline}}$, and let $\pi^+$ denote the policy
obtained by fine-tuning on the positive-advantage dataset
$\mathcal D^+$.

% We consider the state coverage of the filtered dataset,
% \begin{equation}
%     \mathcal S^+
%     =
%     \left\{
%     s:
%     (s,a^{i_m})\in\mathcal D^+
%     \text{ for some }a^{i_m},\ \forall m
%     \right\}.
% \end{equation}

Our analysis relies on the following assumptions:
(1) the critic gives accurate advantages $A^{i_m}_\pi$ on the states visited by $\pi^+$, and $\pi^+$ operates in the support of ${\mathcal D}^+$;
(2) an action kept by the filter has non-negative expected advantage when the preceding agents $i_{1:m-1}$ act according to $\pi^+$.

\begin{proposition}[Monotonic Improvement of Credit-Filtered Tuning]
\label{prop:monotonic_improvement}
Under the assumptions above, the policy $\pi^+$ obtained by
credit-filtered tuning satisfies
\begin{equation}
    J(\pi^+) - J(\pi) \geq 0.
\end{equation}
\end{proposition}

By the performance-difference lemma,
\begin{equation}
\label{eq:performance_difference}
J(\pi^+)-J(\pi)
=
\frac{1}{1-\gamma}
\mathbb E_{
    s\sim d^{\pi^+},
    \mathbf a\sim\pi^+
}
\left[
    A_\pi(s,\mathbf a)
\right].
\end{equation}
Using the multi-agent advantage decomposition lemma
\citep{kuba2022trustregion}, the joint advantage can be decomposed as
\begin{equation}
\label{eq:joint_advantage_decomposition}
A_\pi(s,\mathbf a)
=
\sum_{m=1}^{N}
A_\pi^{i_m}
\left(
s,
\mathbf a^{i_{1:m-1}},
a^{i_m}
\right).
\end{equation}
Therefore,
\begin{equation}
\begin{aligned}
J(\pi^+)-J(\pi)
=
\frac{1}{1-\gamma}
\mathbb E_{s\sim d^{\pi^+}}
\sum_{m=1}^{N}
\mathbb E_{\mathbf a\sim\pi^+}
\Big[
A_\pi^{i_m}
(
s,\mathbf a^{i_{1:m-1}},a^{i_m}
)
\Big].
\end{aligned}
\end{equation}

By construction, credit-filtered tuning retains only actions with
positive agent-wise advantage. Under the exact projection and the assumption that an action kept by the filter has non-negative expected advantage with $\mathbf a^{i_{1:m-1}} \sim \mathbf \pi^{+,i_{1:m-1}}$, $\pi^{+,i_m}$ is therefore supported only on
actions satisfying
\begin{equation}
\mathbb E_{\mathbf a\sim\pi^+} [
A_\pi^{i_m}
(
s,\mathbf a^{i_{1:m-1}},a^{i_m}
)]
\geq 0
\end{equation}
for states covered by $\mathcal D^+$. Hence, every agent-wise
expectation in the above decomposition is non-negative. Since
 $\pi^+$ operates in the support of ${\mathcal D}^+$, summing over agents and integrating over
the state-visitation distribution gives
\begin{equation}
J(\pi^+) - J(\pi) \geq 0,
\end{equation}
which proves the proposition.

\textbf{Remark.} 
(1) The coverage requirement in Assumption 1 is consistent with our initialization-aware data collection, which supplements policy rollouts with human demonstrations for zero-success initializations, thereby broadening the coverage of successful behaviors and reducing the need for extrapolation.
(2) Assumption 2 is plausible when fine-tuning induces moderate policy changes and retained actions have a positive advantage margin. Since fine-tuning starts from $\pi$ and reinforces selected behaviors from the collected data, these behaviors may remain beneficial under the updated policies of preceding agents, provided that the resulting distribution shift is not large enough to reverse their expected advantages.

\section{Implementation Details}
\label{sec:implementation_details}
\subsection{Decentralized execution}
We adopt a decentralized execution setting similar to CHORUS~\citep{doshi2026chorus}. Each VLA receives its local observations, including wrist-camera RGB images, proprioceptive states, and language instructions, and outputs actions only for the end effector of its corresponding robot. Different from CHORUS, where each agent operates solely on its local observations, we additionally provide all agents with a shared third-person RGB observation and the language instruction, since wrist-camera observations alone provide limited global context for effective multi-agent coordination. %We do not share parameters between VLAs since the task require different individual skills and our pilot study do not find consistent performance gains.
As for the language instruction, for RoboTwin~\citep{chen2025robotwin2}, we use the shared instruction provided by the RoboTwin itself. For RoboFactory and real-world tasks, we use the instruction as below: 
\begin{lstlisting}[style=promptstyle, caption={Language provided to the agents in different tasks.}]
RoboFactory Simulation:

Camera Alignment: 
Agent0:  "Grasp one side of the camera, lift the camera and align the camera with the steak.",
Agent1:  "Grasp one side of the camera, lift the camera and align the camera with the steak.",
Agent2:  "Grasp the steak, lift the steak and stay still.".

Three Robots Stack Cubes:
Agent0:  "Lift the blue cube, place the blue cube to the target and stay still.",
Agent1:  "Lift the green cube, stay still and place the green cube on the blue one.",
Agent2:  "Lift the red cube, stay still and place the red cube on the green one.".

Take Photo:
Agent0:  "Grasp one camera handle, lift one camera handle, and stay still.",
Agent1:  "Grasp one camera handle, lift one camera handle, and stay still.",
Agent2:  "Grasp the steak, lift the steak align the camera, and stay still.",
Agent3:  "After the steak align the camera, press the camera shutter.".

Long Pipeline Delivery:
Agent0:  "Grasp the shoe and place it to the target.",
Agent1:  "Grasp the shoe and pass it to the next robotic arm.",
Agent2:  "Grasp the shoe and pass it to the next robotic arm.",
Agent3:  "Grasp the shoe and pass it to the next robotic arm.".

Real-world tasks (shared instructions):
Handover:
Use the left arm to grasp the green-and-white tissue pack, hand it over to the right arm, and place it down on the table.

Lift:
Use both arms to grasp opposite sides of the black plastic mesh storage basket and lift it up.

Handover & Drop:
Use the right arm to grasp the empty a2 milk carton, hand it over to the left arm, and drop it into the white trash bin on the left.
\end{lstlisting}

\subsection{Training details}
For each simulation task, we collect 50 expert demonstrations for supervised fine-tuning and 30 demonstrations for real-world tasks. We trained the CHORUS models based on these demonstrations. During the initialization-aware data collection stage, we evaluate 96 environment initialization configurations for each simulation task, with $K=5$ the number of rollouts for each configuration. While in real-world tasks, we evaluate 50 environment initialization configurations with $K = 1$. % Different from the simulation, we do not collect the whole trajectories of human demonstrations. We allow a human operator to intervene and take over control when an imminent task failure is identified, following a setting similar to DAgger~\citep{ross2010DAgger,zhong2026egosteer}. 

During offline credit-filtered tuning, we use the collected rollouts to train the value function. We adopt $\gamma=1$ to calculate the discounted retures. We use a sparse reward with a step penalty of $-1$. At the final timestep of each episode, the agent receives a reward of $0$ if the task succeeds and a penalty of $-1000$ if the task fails. The architecture of value function largely follows that of $\pi_0$, where agent actions are encoded in the same manner as the noisy actions in $\pi_0$. We additionally introduce a value-expert head that discretizes the Monte Carlo return into 201 bins for value prediction. The value function is trained following the procedure of $\pi_{0.6}^*$~\citep{physicalintelligence2025pistar06}. For the agent ordering $i_1,\ldots,i_{\mathcal N}$, we use a fixed order throughout training and evaluation. Specifically, we use the order 'left, right' for RoboTwin and real-world tasks, while for RoboFactory, we follow the agent ordering defined in the original dataset~\citep{qin2025robofactory}. We then compute the agent-wise advantages and filter out samples whose advantages fall below a threshold $\epsilon$. We choose $\epsilon$ such that approximately 5\% of the original data is retained for subsequent fine-tuning. We add all of the frames from human demonstrations to $\mathcal D^+$. As for the finetuning of the CHORUS model, offline-tuned VLA policies, and the value function, we follow the pipeline similar to $\pi_0$~\citep{black2024pi0} and $\pi_{0.5}$~\citep{physicalintelligence2025pi05}, with the training settings shown in Tab.~\ref{tab:training_settings_vla_value}:
\begin{table}[H]
    \centering
    \vspace{-0.2in}
    \caption{Training settings of the VLA policies and value functions.}
    \label{tab:training_settings_vla_value}
    \begin{tabular}{lc}
        \toprule
        \textbf{Setting} & \textbf{Value} \\
        \midrule
        \textbf{Batch size}     & 32 \\
        \textbf{Training steps} & 30,000 \\
        \textbf{VLM}            & Full-parameter \\
        \textbf{Action expert}  & Full-parameter \\
        \textbf{Chunk horizon}  & 50 for simulation and 20 for real-world \\
        \bottomrule
    \end{tabular}
    \vspace{-0.2in}
\end{table}

For online latent-space reinforcement learning, we train for 200 exploration episodes and select the checkpoint with the best performance for evaluation in simulation. As for the real-world tasks, we change the online RL stage of DSRL to offline version with 50 real-world rollouts as proof-of-concept result. The parameters of DSRL are shown in Tab.~\ref{tab:dsrl_hyperparameters}. The parameters of Flow-Noise and Flow-SDE are shown in Tab.~\ref{tab:rlinf_hyperparameters}.

For evaluation, we evaluate each task using 100 rollouts in simulation and 50 rollouts in the real world. In simulation, we use an action chunk size of 50, with execution horizons of 50 and 20 for RoboTwin and RoboFactory, respectively. For real-world deployment, we employ Real-Time Chunking (RTC)~\citep{black2025rtc} for smooth and responsive VLA control, with an action chunk size of 50, an execution horizon of 8, and an overlap of 12 actions between consecutive action chunks.

\begin{table}
    \centering
    \caption{Training hyperparameters of DSRL.}
    \label{tab:dsrl_hyperparameters}
    \resizebox{\textwidth}{!}{
    \begin{tabular}{lc|lc|lc}
        \toprule
        \textbf{Hyperparameter} & \textbf{Value} &
        \textbf{Hyperparameter} & \textbf{Value} &
        \textbf{Hyperparameter} & \textbf{Value} \\
        \midrule

        Batch size              & 16
        & Actor lr              & $1\times10^{-4}$
        & Critic lr             & $3\times10^{-4}$ \\

        Training episodes       & 200
        & Temperature lr        & $3\times10^{-4}$
        & UTD ratio             & 1 \\

        Hidden dimensions       & $(128,128,128)$
        & Latent dimension      & 50
        & Discount factor       & 0.999 \\

        Target update rate      & 0.005
        & Number of Q-functions & 10
        & Critic reduction      & Mean \\

        CNN features            & $(32,32,32,32)$
        & CNN strides           & $(2,1,1,1)$
        & CNN padding           & VALID \\

        Encoder type            & Small
        & Encoder normalization & Group
        & Bottleneck            & True \\

        Spatial softmax         & True
        & Dropout rate          & 0.0
        & Target entropy        & Auto \\

        Action magnitude        & 1.0
        & Number of cameras     & 1
        & PopArt                & False \\

        \bottomrule
    \end{tabular}}
\end{table}

\begin{table*}
    \centering
    \caption{Training hyperparameters of Flow-SDE and Flow-Noise.}
    \label{tab:rlinf_hyperparameters}
    \resizebox{\textwidth}{!}{
    \begin{tabular}{lc|lc|lc}
        \toprule
        \textbf{Hyperparameter} & \textbf{Value} &
        \textbf{Hyperparameter} & \textbf{Value} &
        \textbf{Hyperparameter} & \textbf{Value} \\
        \midrule
        Training episodes & 200 &
        Advantage type & embodied\_grpo &
        Train expert only & True \\

        Top-k & 50 &
        Loss type & embodied\_grpo &
        Reward type & chunk\_level \\

        Parallel Num & 8 &
        PPO clip high & 0.2 &
        Separate models & True \\

        Discount factor & 0.99 &
        PPO clip low & 0.2 &
        Action chunks & 50 \\

        Chunk steps & 16 &
        PPO clip c & 3.0 &
        Flow steps & 10 \\

        Eval chunk steps & 16 &
        Actor lr & $5\times10^{-6}$ &
        Noise level & \makecell{0.1 (Flow-SDE) \\ 0.01 (Flow-Noise)} \\

        Normalize advantages & True &
        Micro-batch size & 16 &
        Joint log-prob & True \\

        KL coefficient & 0.0 &
        Global batch size & 128 &
        Train temperature & 1.0 \\

        Entropy bonus & 0 &
        Max grad norm & 2.0 &
        Eval temperature & 0.6 \\
        \bottomrule
    \end{tabular}
    }
\end{table*}

\end{document}